\documentclass[10pt,twocolumn,letterpaper]{article}
\usepackage{cvpr}
\usepackage{times}
\usepackage{epsfig}
\usepackage{graphicx}
\usepackage{amsmath,amssymb,bm}
\usepackage{verbatim}
\usepackage[svgnames]{xcolor}
\usepackage{algorithm}
\usepackage{algpseudocode}
\usepackage{graphicx}
\usepackage{subcaption}
\usepackage{gensymb}
\usepackage{amsthm}
\usepackage{authblk}

\usepackage[pagebackref=true,breaklinks=true,letterpaper=true,colorlinks,bookmarks=false]{hyperref}

\DeclareMathOperator*{\argmax}{arg\,max}

\newtheorem{theorem}{Theorem}[section]

\newtheorem{lemma}[theorem]{Lemma}

\cvprfinalcopy 

\ifcvprfinal\fi
\begin{document}

\title{End-to-end Training of CNN-CRF via Differentiable Dual-Decomposition}

\author[1]{Shaofei Wang}
\author[2]{Vishnu Lokhande}
\author[3]{Maneesh Singh}
\author[1]{Konrad Kording}
\author[3]{Julian Yarkony}
\affil[1]{University of Pennsylvania}
\affil[2]{University of Wisconsin-Madison}
\affil[3]{Verisk Computational and Human Intelligence Laboratory}


\maketitle

\begin{abstract}
Modern computer vision (CV) is often based on convolutional neural networks (CNNs) that excel at hierarchical feature extraction. The previous generation of CV approaches was often based on conditional random fields (CRFs) that excel at modeling flexible higher order interactions.  As their benefits are complementary they are often combined.  However, these approaches generally use mean-field approximations and thus, arguably, did not directly optimize the real problem. Here we revisit dual-decomposition-based approaches to CRF optimization, an alternative to the mean-field approximation. These algorithms can efficiently and exactly solve sub-problems and directly optimize a convex upper bound of the real problem, providing optimality certificates on the way.  Our approach uses a novel fixed-point iteration algorithm which enjoys dual-monotonicity, dual-differentiability and high parallelism.  The whole system, CRF and CNN can thus be efficiently trained using back-propagation. 
We demonstrate the effectiveness of our system on semantic image segmentation, showing consistent improvement over baseline models.
\end{abstract}

\section{Introduction}
\label{sec:intro}
The end-to-end training of systems that combine CNNs and CRFs is a popular research direction.  This combination often improves quality in pixel-labeling tasks in comparison with decoupled training~\cite{CRFasRNN,Knobelreiter_2017_CVPR}.  Most of the existing end-to-end trainable systems are based on a mean-field (MF) approximation to the CRF~\cite{CRFasRNN, DPN, piecewiseCRF, Song_2019_ICCV}.  The MF approximation approximates the posterior distribution of a CRF via a set of variational distributions, which are of simple forms and amenable to analytical solution.  This approximation is popular in end-to-end trainable frameworks where the CRF is combined with a CNN because MF iterations can be unrolled as a set of recurrent convolutional and arithmetic layers (\cf\ \cite{CRFasRNN, piecewiseCRF, DPN}) and thus is fully-differentiable.  Despite the computational efficiency and easy implementation, MF based approaches suffer from the somewhat bold assumptions that the underlying latent variables are independent and the variational distributions are simple.  The exact maximum-a-posteriori (MAP) solution of a CRF can thus never be attained with MF iterations, since in practical CRFs latent variables usually are not independent and the posterior distributions are complex and can not be analytically expressed.  For example, in order to employ efficient inference,~\cite{CRFasRNN, densecrf} model the pairwise potentials as the weighted sum of a set of Gaussian kernels over pairs of feature vectors, this will always penalize or give very small boost to dissimilar feature vectors, thus it tends to smooth-out pixel-labels spatially.  On the other hand, a more general pairwise model should be able to encourage assignment of different labels to a pair of dissimilar feature vectors, if they actually belong to different semantic classes in ground-truth.

Here we explore an alternative, and historically popular solution to MAP inference of Markov random field (MRF), the dual-decomposition~\cite{komodakis2007mrf,sontag2011introduction}.  Note that CRFs are simply MRFs conditioned on input data, thus inference methods on MRFs are applicable to CRFs if input data and potential function are fixed per inference.  Dual-decomposition does not make any assumptions about the distribution of CRF but instead formulate the MAP inference problem as an energy maximization \footnote{Forming MAP inference as energy minimization problem is also common in the literature, here we choose maximization in order to stay consistent with typical classification losses, \eg\ Cross-Entropy with Softmax} problem.  Directly solving such problems on graph with cycles is typically NP-hard~\cite{MAP_NPhard}, however the dual-decomposition approach relaxes the original problem by decomposing the graph into a set of trees that cover each edge and vertex at least once; MAP inference on these tree-structured sub-problems can be done efficiently via dynamic programming, while the solution to the original problem can be attained via either dual-coordinate descent or sub-gradient descent using solutions from the sub-problems.  Dual-decomposition is still an approximate algorithm in that it minimizes a convex upper-bound of the original problem and the solutions of the sub-problems do not necessarily agree with each other (even if they do agree, it could still be a fixed point).  However the approximate primal objective can be attained via heuristic decoding anytime during optimization and the duality-gap describes the quality of the current solution, and if the gap is zero then we are guaranteed to have found an optimal solution.  In comparison, the MF approximation only guarantees a local minimum of the Kullback Leibler (KL) divergence.

When it comes to learning parameters for CNNs and CRFs, the MF approximation is fully differentiable and thus trainable with back-propagation.  Popular dual-decomposition approaches, however, rely on either sub-gradient descent or dual-coordinate-descent to maximize the energy objectives and thus are not immediately differentiable with respect to CNN parameters that generate the CRF potentials.  Max-margin learning is typically used in such a situation for linear models~\cite{Tsochantaridis2005SSVM, komodakis2011m3, Finley2008SSVM} and non-linear, deep-neural-network models~\cite{Chen2015Structured, Knobelreiter_2017_CVPR}. However, they require the MAP-inference routine to be robust enough to find reasonable margin-violators to enable learning. This is especially problematic when the underlying CRF potentials are parameterized by non-linear, complex CNNs instead of linear classifiers as in traditional max-margin frameworks (\eg\ structured-SVM), as argued by~\cite{belanger2017SPEN}.  In contrast, our proposed fixed-point iteration, which is derived from a special-case of more general block coordinate descent algorithms~\cite{sontag2009tbcd, globerson2008mplp}, optimizes the CRF energy directly with respect to CRF potentials and thus can be jointly trained with CNNs via back-propagation; our work shares similar intuition as~\cite{belanger2017SPEN} that max-margin learning can be unstable when combined with deep learning, but instead of employing a gradient-predictor for MAP-inference which has little theoretical guarantees, we derive a differentiable dual-decomposition algorithm for MAP-inference that has desirable theoretical properties such as dual-monotonicity and provable optimality.

The major contribution of our work is thus threefold:
\begin{enumerate}
    \item We revisit dual-decomposition and derive a novel node-based fixed-point iteration algorithm that enjoys dual-monotonicity and dual-sub-differentiability, and provide an efficient, highly-parallel GPU implementation for this algorithm.
    \item We introduce the smoothed-max operator~\cite{DDP} to make our fixed-point iteration algorithm fully-differentiable while remaining monotone for the dual objective, and discuss the benefits of using it instead of typical $max$ operator in practice.
    \item We demonstrate how to conduct end-to-end training of CNNs and CRFs by differentiating through the dual-decomposition layers, and show improvements over pure CNN models on a semantic segmentation task on the PASCAL-VOC 2012 dataset\cite{pascal-voc-2012}.
\end{enumerate}

%
%

\section{Joint Modeling of CRFs and CNNs}
\label{sec:CRFs}
Given an image $\mathcal{I}$ with size $M \times N$, we want to conduct per-pixel labeling on the image, producing a label $l \in \mathcal{L}$ for each pixel in the image (\eg\ semantic classes or disparity).  A CRF aims to model the conditional distribution $P(\bm{L} | \mathcal{I})$ with $\bm{L} \in \mathcal{L}^{|M \times N|}$.  When modeled jointly with a CNN $f$ parameterized by $\bm{\theta}$, the conditional distribution is often written as:

\begin{align}
\label{eqn:CRF_posterior}
P(\bm{L} | \mathcal{I}) = \frac{1}{Z} \exp( f_{\bm{\theta}}(\bm{L} ; \mathcal{I}) )
\end{align}

where $Z = \sum_{\bm{L} \in \mathcal{L}^{|M \times N|}} \exp( f_{\bm{\theta}}(\bm{L} ; \mathcal{I}) )$ is a normalizing constant that does not depend on $\bm{L}$.  Following common practices, in this paper we only consider pairwise potential functions:

\begin{align}
\label{eqn:CRF_potential}
f_{\bm{\theta}}(\bm{L};\mathcal{I}) = \sum_{i \in V} \psi_{\bm{\theta}}(l_i; \mathcal{I}) + \sum_{ij \in E} \phi_{\bm{\theta}}(l_i, l_j; \mathcal{I})
\end{align}

where $\psi(\cdot)$ and $\phi(\cdot, \cdot)$ are neural networks modeling unary and pairwise potential functions. $V$ denotes the set of all pixel locations and $E$ denotes the pairwise edges in the graph.  Finding the mode (or maximum of the modes in multi-modal case) of the posterior distribution Eq.~\eqref{eqn:CRF_posterior} is equivalent to finding the maximizing configuration $\bm{L}^*$ of $\bm{L}$ to Eq.~\eqref{eqn:CRF_potential}, that is:

\begin{align}
\label{eqn:CRF_argmax}
\bm{L}^* &= \argmax_{\bm{L}} f_{\bm{\theta}} (\bm{L};\mathcal{I}) \\
P(\bm{L}^* | \mathcal{I}) &= \frac{1}{Z} \exp( \max_{\bm{L}} f_{\bm{\theta}} (\bm{L} ; \mathcal{I}) )
\end{align}

When conducting test time optimization, we aim to maximize the objective function $f_{\bm{\theta}} (\bm{L} ; \mathcal{I})$ w.r.t. $\bm{L}$, thus the optimization problem can be written as:

\begin{align}
\label{eqn:MAP_inference}
\max_{\bm{L}} f_{\bm{\theta}}(\bm{L};\mathcal{I})
\end{align}

Eq.~\eqref{eqn:MAP_inference} is the MAP inference problem.  We show how to solve this problem via dual-decomposition in the next section.

\section{Dual-Decomposition for MAP Inference}
\label{sec:LP-MAP}
We formally define a graph $G = (V,E)$ with $M \times N$ vertices representing a 2D-grid.  Each vertex can choose one of the states in the label set $\mathcal{L} = \{1, 2, \dots , L \}$. We define a labeling of the grid as $\bm{L} \in \mathcal{L}^{|M \times N|}$ and the state of vertex at location $i$ as $l_i$.  For simplicity, we will drop the dependency on $\bm{\theta}$ and $\mathcal{I}$ for all potential functions $f$, $\phi$ and $\psi$ during the derivation of inference algorithm, since they are fixed per inference.  This MAP inference problem on the MRF is defined as:

\begin{align}
\label{eqn:MRF-MAP_inference}
\max_{\bm{L}} \sum_{i \in V} \psi(l_i) + \sum_{ij \in E} \phi(l_i, l_j)
\end{align}

\subsection{Integer Linear Programming Formulation to MAP Problem}
To derive the dual-decomposition we first transform Eq.~\eqref{eqn:MRF-MAP_inference} into an integer linear programming (ILP) problem.  Let us denote $\bm{x}_i(l)$ as the distribution corresponding to vertex/location $i$ and $\bm{x}_{ij}(l', l)$ as the joint distribution corresponding to a pair of vertices/locations $i, j$.  The constraint set $\mathcal{X}^{\mathcal{G}}$ is defined to enforce the pairwise and unary distribution to be consistent and discrete:

\begin{equation}
\mathcal{X}^{\mathcal{G}} = 
  \left\{ 
    \bm{x} \; \left| \;
    \begin{aligned}
      &\sum\nolimits_{l} \bm{x}_i(l) = 1, \quad \forall i \in V\\
      &\sum\nolimits_{l'} \bm{x}_{ij}(l, l') = \bm{x}_i(l), \quad \forall (ij, l) \in E \times \mathcal{L}\\
      &\bm{x}_i(l) \in \{0, 1\}, \bm{x}_{ij}(l, l') \in \{0, 1\}
    \end{aligned} \nonumber \right.
  \right\}
\end{equation}

Now we re-write Eq.~\eqref{eqn:MRF-MAP_inference} as an ILP as:

\begin{align}
\label{eqn:MAP_ILP}
\max_{\bm{x}} \quad & \sum_{i \in V} \bm{\psi}_i \cdot \bm{x}_i + \sum_{ij \in E} \bm{\phi}_{ij} \cdot \bm{x}_{ij} \\
\text{s.t.} \quad & \bm{x} \in \mathcal{X}^{\mathcal{G}} \nonumber
\end{align}

where $\bm{x}_i$ and $\bm{\psi}_i$ are $|\mathcal{L}|$ dimensional vectors representing vertex distribution and scores at vertex $i$, while $\bm{x}_{ij}$ and $\bm{\phi}_{ij}$ are $|\mathcal{L}|^2$ dimensional vectors representing edge distribution and scores for a pair of vertex $i,j$.

\subsection{Dual-Decomposition for Integer Linear Programming}

\begin{figure*}[t]
    \centering
    \includegraphics[width=0.95\textwidth]{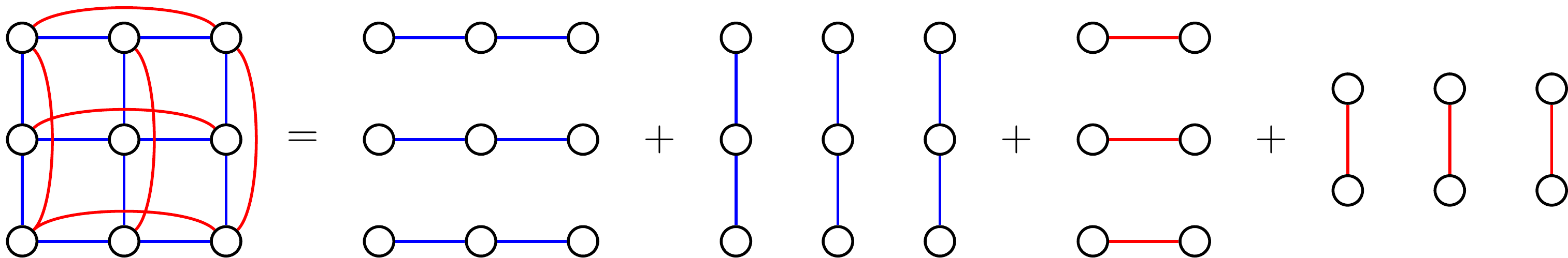} \\
    \caption{Illustration for decomposition of a grid-graph with long-range interactions into sets of horizontal and vertical chains.  We use stride 1 and stride 2 edges, as illustrated in this figure, for all of our models, but our derivation in Sec.~\ref{sec:monotone_updates} holds for general graphs with arbitrary connections.}
    \label{fig:graph_decomp_illust}
\end{figure*}

Given an arbitrary graph with cycles, solving Eq.~\eqref{eqn:MAP_ILP} is  generally NP-hard.  Dual-decomposition~\cite{komodakis2007mrf,sontag2011introduction} tackles Eq.~\eqref{eqn:MAP_ILP} by first decomposing the graph $G$ into a set of tree-structured, easy-to-solve sub-problems, such that each edge and vertex in $G = (V, E)$ is covered at least once by these sub-problems. A set of additional constraints is then added to enforce that maximizing configurations of each sub-problem agree with one another. These constraints are then relaxed by Lagrangian relaxation and the relaxed objective can be optimized by either sub-gradient ascent~\cite{komodakis2007mrf,komodakis2011m3} or fixed-point updates~\cite{sontag2009tbcd,yarkony2010covering,globerson2008mplp,TRW}.

Formally, we define a set of trees $\mathcal{T}$ that cover each $i \in V$ and $ij \in E$ at least once.  We denote the set of variables corresponding to vertices and edges in tree $t \in \mathcal{T}$ and all sets of such variables as $\bm{x}^t$ and $\{ \bm{x}^t \}$, respectively.  We use $\mathcal{T}(i)$ and $\mathcal{T}(ij)$ to denote the set of trees that cover vertex $i \in V$ and edge $ij \in E$, respectively.  With the above definitions, Eq.~\eqref{eqn:MAP_ILP} can be rewritten as: 

\begin{align}
\label{eqn:MAP_dual_decomposition}
\max_{\{\bm{x}^t\}, \bm{x}} \quad & \sum_{t \in \mathcal{T}} \left( \sum_{i \in V} \bm{x}_i^t \cdot \bm{\psi}_i^t + \sum_{ij \in E} \bm{x}_{ij}^t \cdot \bm{\phi}_{ij}^t \right) \\
\text{s.t.} \quad & \bm{x}^t \in \mathcal{X}^{\mathcal{G}}, \quad \forall t \in \mathcal{T} \nonumber \\
& \bm{x}^t_i = \bm{x}_i, \quad \forall i \in V, t \in \mathcal{T}(i)  \nonumber \\
& \bm{x}^t_{ij} = \bm{x}_{ij}, \quad \forall ij \in E, t \in \mathcal{T}(ij)  \nonumber
\end{align}

For our application, we decompose the graph with vertical and horizontal connections of arbitrary length into sets of horizontal and vertical chain sub-problems (see Fig.~\ref{fig:graph_decomp_illust} for example).  In such a case, each node is covered multiple times while each edge is covered exactly once, we replicate their scores by the number of times they are covered by sub-problems. The replicated scores $\{ \bm{\psi}^t, \bm{\phi}^t \}$ should satisfy:

\begin{align}
\label{eqn:dual_decomposition_scores}
\sum_{t \in T(i)} \bm{\psi}^t_i = \bm{\psi}_i, \quad, \bm{\phi}^t_{ij} = \bm{\phi}_{ij}
\end{align}

We then have no replicated edge anymore, and can re-write Eq.~\eqref{eqn:MAP_dual_decomposition} as:

\begin{align}
\label{eqn:MAP_dual_decomposition_chains}
\max_{\{\bm{x}^t\}, \bm{x}} \quad & \sum_{t \in \mathcal{T}} \left( \sum_{i \in V} \bm{x}_i^t \cdot \bm{\psi}_i^t + \sum_{ij \in E} \bm{x}_{ij}^t \cdot \bm{\phi}_{ij} \right) \\
\text{s.t.} \quad & \bm{x}^t \in \mathcal{X}^{\mathcal{G}}, \quad \forall t \in \mathcal{T} \nonumber \\
& \bm{x}^t_i = \bm{x}_i, \quad \forall i \in V, t \in \mathcal{T}(i)  \nonumber
\end{align}

Applying Lagrangian multipliers to relax the second set of constraints (\ie\ agreement constraints among sub-problems), we have:

\begin{align}
\begin{split}
\min_{\{ \bm{\lambda}^t \}} \max_{\{\bm{x}^t\}, \bm{x}} \quad & \sum_{t \in \mathcal{T}} \left( \sum_{i \in V} \bm{x}_i^t \cdot (\bm{\psi}_i^t + \bm{\lambda}_i^t) + \right. \\
& \quad \quad \left. \sum_{ij \in E} \bm{x}_{ij}^t \cdot \bm{\phi}_{ij} - \sum_{i \in V} \bm{x}_i \cdot \bm{\lambda}_i^t \right) \label{eqn:MAP_dual_decomposition_dual} \\
\text{s.t.} \quad& \bm{x}^t \in \mathcal{X}^{\mathcal{G}}, \quad \forall t \in \mathcal{T}
\end{split}
\end{align}

Where $\bm{\lambda}^t$ and $\{ \bm{\lambda}^t \}$ denote dual variables for sub-problem $t$ and the set of all sub-problems, respectively.  $\bm{\lambda}^t_i$ denotes dual variables for sub-problem $t$ at location $i$ and has dimension of $L$ (\ie\ number of labels/states).  It is easy to check that if $\sum_{t \in \mathcal{T}(i)} \bm{\lambda}_i^t \neq \bm{0}$ for any $i \in V$, then $\bm{\lambda}_i^t = +\infty$.  Thus we must enforce $\sum_{t \in \mathcal{T}(i)} \bm{\lambda}_i^t = 0, \forall i \in V$,  it follows that $\sum_{i \in V} \sum_{t \in \mathcal{T}} \bm{x}_i \cdot \bm{\lambda}_i^t = 0$, and we can therefore eliminate $\{ \bm{x} \}$ from Eq.~\eqref{eqn:MAP_dual_decomposition_dual}, resulting in:

\begin{align}
\label{eqn:MAP_dual_decomposition_final}
\min_{\{ \bm{\lambda}^t \}} \max_{\{\bm{x}^t\}} \quad & \sum_{t \in \mathcal{T}} \left( \sum_{i \in V} \bm{x}_i^t \cdot (\bm{\psi}_i^t + \bm{\lambda}_i^t) + \sum_{ij \in E} \bm{x}_{ij}^t \cdot \bm{\phi}_{ij} \right) \\
\text{s.t.} \quad & \bm{x}^t \in \mathcal{X}^{\mathcal{G}}, \quad \forall t \in \mathcal{T} \nonumber \\
& \sum_{t \in \mathcal{T}(i)} \bm{\lambda}^t_i = 0, \quad \forall i \in V \nonumber
\end{align}

\subsection{Monotone Fixed-Point Algorithm for Dual-Decomposition}
\label{sec:monotone_updates}
In this section we derive a block coordinate-descent algorithm that monotonically decreases the objective of Eq.~\eqref{eqn:MAP_dual_decomposition_final}, generalizing~\cite{sontag2009tbcd,yarkony2010covering}.  It is often convenient to initialize $\bm{\lambda}$'s as 0 and fold them into $\{ \bm{\psi}^t \}$ terms such that we optimize an equivalent objective to Eq.~\eqref{eqn:MAP_dual_decomposition_final} over $\{ \bm{\psi}^t \}$:

\begin{align}
\label{eqn:MAP_dual_decomposition_reparam}
\min_{\{ \bm{\psi}^t \}} \max_{\{\bm{x}^t\}} \quad & \sum_{t \in \mathcal{T}} \left( \sum_{i \in V} \bm{x}_i^t \cdot \bm{\psi}_i^t + \sum_{ij \in E} \bm{x}_{ij}^t \cdot \bm{\phi}_{ij} \right) \\
\text{s.t.} \quad & \bm{x}^t \in \mathcal{X}^{\mathcal{G}}, \quad \forall t \in \mathcal{T} \nonumber \\
& \sum_{t \in \mathcal{T}(i)} \bm{\psi}^t_i = \bm{\psi}_i, \quad \forall i \in V \nonumber
\end{align}

Now consider fixing the dual variables for all sub-problems at all locations except for those at one location $k$ and optimizing only with respect to the vector $\bm{\psi}^t_{k}, \forall t \in \mathcal{T}(k)$ and primal variables $\{ \bm{x}^t \}$.
Define $\bm{\mu}_k^t(l) = \bm{\psi}_k^t(l) + \max_{\substack{\bm{x}^t \in \mathcal{X}^{\mathcal{G}} \\ \bm{x}_k^t(l) = 1}} \sum_{i \in V \setminus k} \bm{x}_i^t \cdot \bm{\psi}_i^t + \sum_{ij \in E} \bm{x}_{ij}^t \cdot \bm{\phi}_{ij}$ as the max-marginal of sub-problem $t$ at location $k$ with $\bm{x}^t_k(l) = 1$, similarly we define the max-marginal vector of sub-problem $t$ at location $k$ as $\bm{\mu}^t_k$ , and the max-energy of sub-problem $t$ as $\bm{\mu}^t$.  Note that $\bm{\mu}^t_k$ is a vector-value while $\bm{\mu}_k^t(l)$ and $\bm{\mu}^t$ are scalar-values.

\begin{lemma}
\label{lemm_1}
For a single location $k \in V$, the following coordinate update to $\bm{\psi}^t_{k}, \forall t \in \mathcal{T}(k)$ is optimal:
\begin{align}
\label{eqn:update_rule_single}
 \bm{\psi}_k^t(l_k) \leftarrow \bm{\psi}_k^t(l_k) - \left(\bm{\mu}_k^t(l_k) - \frac{1}{|\mathcal{T}(k)|} \sum_{\bar{t} \in \mathcal{T}(k)} \bm{\mu}_k^{\bar{t}}(l_k) \right), \nonumber \\
 \forall t \in \mathcal{T} (k), l_k \in \mathcal{L}
\end{align}
\end{lemma}

\begin{proof}
We want to optimize the following linear program with respect to the dual variables (note that the primal variables $\{ \bm{x}^t \}$ are included in the max-energy terms $\bm{\mu}^t$)

\begin{align}
\label{eqn:LP_dual_coordinate_descent}
\min_{\substack{\bm{\psi}^t_k \in \mathbb{R}^{|\mathcal{L}|} \\ \bm{\mu}^t \in \mathbb{R}^{|\mathcal{L}|} \\ \forall t \in \mathcal{T}(k)}} \quad & \sum_{t \in \mathcal{T}(k)} \bm{\mu}^t \\
\text{s.t.} \quad & \hat{\bm{\mu}}^t_k(l_k) + \bm{\psi}^t_k(l_k) - \hat{\bm{\psi}}^t_k(l_k) \leq \bm{\mu}^t, \forall t \in \mathcal{T}(k), \forall l_k \in \mathcal{L} \nonumber \\
& \sum_{t \in \mathcal{T}(k)} \bm{\psi}^t_k(l_k)= \bm{\psi}_k(l_k), \quad \forall l_k \in \mathcal{L} \nonumber
\end{align}

where $\hat{\bm{\mu}}^t_k(l_k)$ and $\hat{\bm{\psi}}^t_k(l_k)$ in the first set of constraints that are max-marginals and unary potentials at location $k$ after applying an update rule, \eg\ Eq.~\eqref{eqn:update_rule_single}.  This set of constraints is derived from the fact that $\hat{\bm{\mu}}^t_k(l_k) + \bm{\psi}^t_k(l_k) - \hat{\bm{\psi}}^t_k(l_k) = \bm{\mu}^t_k(l_k), \forall l_k \in \mathcal{L}$ and $\bm{\mu}^t = \max_{l_k \in \mathcal{L}} \bm{\mu}^t_k(l_k)$.

Converting Eq.~\eqref{eqn:LP_dual_coordinate_descent} to dual form we have:

\begin{align}
\label{eqn:LP_dual_coordinate_descent_dual}
\begin{split}
\max_{\substack{\bm{\alpha}^t \geq 0, \forall t \in \mathcal{T}(k) \\ \bm{\beta}(l_k) \in \mathbb{R}^{|\mathcal{L}|}}} \quad & \sum_{l_k \in \mathcal{L}} \bm{\beta}(l_k) \cdot \bm{\psi}_k(l_k) \\
& + \sum_{t \in \mathcal{T}(k)} \sum_{l_k \in \mathcal{L}} \bm{\alpha}^t(l_k) \left( \hat{\bm{\mu}}^t_k(l_k) - \hat{\bm{\psi}}^t_k(l_k) \right) \\
\text{s.t.} \quad & 1 - \sum_{l_k \in \mathcal{L}} \bm{\alpha}^t(l_k) = 0, \quad \forall t \in \mathcal{T}(k) \\
& \bm{\beta}(l_k) - \bm{\alpha}^t(l_k) = 0, \quad \forall t \in \mathcal{T}(k)
\end{split}
\end{align}

It can be inferred from the second set of constraints that the terms $\bm{\alpha}^t(l_k) = \bm{\alpha}^{\bar{t}}(l_k), \forall (t, \bar{t}) \in \mathcal{T}(k) \times \mathcal{T}(k)$, setting $\bm{\beta}(l_k) = \bm{\alpha}^{\bar{t}}(l_k)$ for any $\bar{t} \in \mathcal{T}(k)$ for all $l_k \in \mathcal{L}$, the optimization becomes:

\begin{align}
\label{eqn:LP_dual_coordinate_descent_dual_simple}
\begin{split}
\max_{\bm{\alpha}^{\bar{t}} \geq 0} \quad & \sum_{l_k \in \mathcal{L}} \bm{\alpha}^{\bar{t}}(l_k) \cdot \left (\bm{\psi}_k(l_k) - \sum_{t \in \mathcal{T}(k)} \hat{\bm{\psi}}^t_k(l_k) \right) \\
& + \sum_{l_k \in \mathcal{L}} \sum_{t \in \mathcal{T}(k)} \bm{\alpha}^{\bar{t}}(l_k) \cdot \hat{\bm{\mu}}^t_k(l_k)   \\
\text{s.t.} \quad & 1 - \sum_{l_k \in \mathcal{L}} \bm{\alpha}^{\bar{t}}(l_k) = 0
\end{split}
\end{align}

Note that update rule Eq.~\eqref{eqn:update_rule_single} always satisfy that $\sum_{t \in \mathcal{T}(k)} \hat{\bm{\psi}}^t_k = \bm{\psi}_k$, thus we can further simplify Eq.~\eqref{eqn:LP_dual_coordinate_descent_dual_simple} as:

\begin{align}
\label{eqn:LP_dual_coordinate_descent_dual_simplest}
\max_{\bm{\alpha}^{\bar{t}} \geq 0} \quad & \sum_{l_k \in \mathcal{L}} \sum_{t \in \mathcal{T}(k)} \bm{\alpha}^{\bar{t}}(l_k) \cdot \hat{\bm{\mu}}^t_k(l_k) \\
\text{s.t.} \quad & 1 - \sum_{l_k \in \mathcal{L}} \bm{\alpha}^{\bar{t}}(l_k) = 0 \nonumber
\end{align}

Note that by applying Eq.~\eqref{eqn:update_rule_single} we have $\hat{\bm{\mu}}^t_k = \hat{\bm{\mu}}^{\bar{t}}_k, \forall (t, \bar{t}) \in \mathcal{T}(k) \times \mathcal{T}(k)$, thus Eq.~\eqref{eqn:LP_dual_coordinate_descent_dual_simplest} = $\max_{l_k \in \mathcal{L}} \sum_{t \in \mathcal{T}(k)} \hat{\bm{\mu}}^t_k (l_k)$ = $\sum_{t \in \mathcal{T}(k)} \hat{\bm{\mu}}^t$.  It is also straightforward to show that by applying Eq.~\eqref{eqn:update_rule_single} to a single location, we have $\sum_{t \in \mathcal{T}(k)} \hat{\bm{\mu}}^t = \sum_{t \in \mathcal{T}(k)} \bm{\mu}^t$, thus Eq.~\eqref{eqn:LP_dual_coordinate_descent} = Eq.~\eqref{eqn:LP_dual_coordinate_descent_dual_simplest}, \ie\ the duality gap of the LP is 0 and update Eq.~\eqref{eqn:update_rule_single} is an optimal coordinate update step.

\end{proof}

In practice we could update \textit{all} $\bm{\psi}$'s at once, which, although it may result in slower convergence, may permit efficient parallel algorithms for modern GPU architectures.  The intuition is that neural-networks can minimize the empirical risk with respect to our coordinate-descent algorithm, since output from any coordinate-descent step is (sub-)differentiable. 

\begin{theorem}
\label{theo_1}
The following update to $\bm{\psi}^t_i, \forall i \in V, t \in \mathcal{T}(i)$ will not increase the objective~\eqref{eqn:MAP_dual_decomposition_reparam}:

\begin{align}
\label{eqn:update_rule_all}
 \bm{\psi}_i^t(l_i) \leftarrow \bm{\psi}_i^t(l_i) - \frac{1}{|V^0|} \left(\bm{\mu}_i^t(l_i) - \frac{1}{|\mathcal{T}(i)|} \sum_{\bar{t} \in \mathcal{T}(i)} \bm{\mu}_i^{\bar{t}}(l_i) \right), \nonumber \\
 \forall i \in V, t \in \mathcal{T} (i), l_i \in \mathcal{L}
\end{align}

where $|V^0| = \max_{t \in \mathcal{T}} |V^t|$ and $|V^t|$ denotes the number of vertices in sub-problem $t$.
\end{theorem}

\begin{proof}
We denote $\hat{\bm{\mu}}^t$ as the max-energy of sub-problem $t$ after we apply update Eq.~\eqref{eqn:update_rule_all} to $\{ \bm{\psi}^t \}$ and, with slight abuse of notations, $\bar{\bm{\mu}}^t_i$ as max-energy of sub-problem $t$ after we apply update Eq.~\eqref{eqn:update_rule_single} for location $i$.  Consider changes in objective from updating $\{ \bm{\psi}^t \}$ according to Eq.~\eqref{eqn:update_rule_all}:

\begin{align}
\label{eqn:change_in_dual_obj}
&- \sum_{t \in \mathcal{T}} \bm{\mu}^t + \sum_{t \in \mathcal{T}} \hat{\bm{\mu}}^t
\end{align}

We briefly expand $\bm{\mu}^t$ as:

\begin{align}
\label{eqn:updated_max_energy}
\bm{\mu}^t = \max_{\bm{x}^t} \quad & \sum_{i \in V} \bm{x}_i^t \cdot \bm{\psi}_i^t + \sum_{ij \in E} \bm{x}_{ij}^t \cdot \bm{\phi}_{ij} \\
\text{s.t.} \quad & \bm{x}^t \in \mathcal{X}^{\mathcal{G}} \nonumber
\end{align}

It is clear that $\bm{\mu}^t$ is a convex function of $\bm{\psi}^t$ because $\bm{\mu}^t$ is the maximum of a set of affine functions (each of which is defined by a point in $\mathcal{X}^{\mathcal{G}}$) of $\bm{\psi}^t$.  When $|V^0| = \max_{t \in \mathcal{T}} |V^t|$ we can apply Jensen's Inequality to the second term of Eq.~\eqref{eqn:change_in_dual_obj}: 

\begin{align}
Eq.~\eqref{eqn:change_in_dual_obj} \leq &- \sum_{t \in \mathcal{T}} \bm{\mu}^t + \sum_{t \in \mathcal{T}} \left( \frac{|V^0| - |V^t|}{|V^0|} \bm{\mu}^t + \sum_{i \in V^t} \frac{1}{|V^0|} \bar{\bm{\mu}}^t_i \right) \nonumber \\
= & \frac{1}{|V^0|} \sum_{t \in \mathcal{T}} \sum_{i \in V^t} \left( \bar{\bm{\mu}}^t_i - \bm{\mu}^t \right) \nonumber \\
= & \frac{1}{|V^0|} \sum_{i \in V^t} \sum_{t \in \mathcal{T}(i)} \left( \bar{\bm{\mu}}^t_i - \bm{\mu}^t \right) \nonumber
\end{align}

Observe that $\sum_{t \in \mathcal{T}(i)} \bar{\bm{\mu}}^t_i$ corresponds to Eq.~\eqref{eqn:LP_dual_coordinate_descent_dual_simplest} while $\sum_{t \in \mathcal{T}(i)} \bm{\mu}^t$ correspond to Eq.~\eqref{eqn:LP_dual_coordinate_descent}.  Since Eq.~\eqref{eqn:LP_dual_coordinate_descent} and Eq.~\eqref{eqn:LP_dual_coordinate_descent_dual_simplest} are primal and dual of an LP we have Eq.~\eqref{eqn:LP_dual_coordinate_descent_dual_simplest} $\leq$ Eq.~\eqref{eqn:LP_dual_coordinate_descent}, thus $\sum_{t \in \mathcal{T}(i)} \left( \bar{\bm{\mu}}^t_i - \bm{\mu}^t \right) \leq 0$, which finally results in Eq.~\eqref{eqn:change_in_dual_obj} $\leq$ 0 and thus the update Eq.~\eqref{eqn:update_rule_all} constitutes a non-increasing step to objective Eq.~\eqref{eqn:MAP_dual_decomposition_reparam}.

\end{proof}

We note that our update rules are novel and different from typical tree-block coordinate descent~\cite{sontag2009tbcd} and its variant~\cite{yarkony2010covering}.  Eq.~\eqref{eqn:update_rule_all} is mostly similar to that in the Fig. 1 of~\cite{sontag2009tbcd}, however we avoid reparameterization of the edge potentials which is expensive when forming coordinate-descent as differentiable layers, as the memory consumption grows linearly with the number of coordinate-descent steps while each step requires $\mathcal{O}(|E| \times |\mathcal{L}|^2)$ memory for storing edge potentials.  Our update rules are also similar to the fixed-point update in~\cite{yarkony2010covering}, however their monotone fixed-point update works on a single covering tree thus is not amenable for parallel computation. Furthermore, for simultaneous update to all locations, we prove a monotone update step-size of $\frac{1}{\max_{t \in T} |V^t|}$ while~\cite{yarkony2010covering} only proves an monotone update step-size of $\frac{1}{|V|}$, which is much smaller.  Our entire coordinate-descent/fixed-point update algorithm is described as Alg.~\ref{alg:DD_fixed-point}

\begin{algorithm}
\caption{Fixed-point Algorithm for Dual-Decomposition}
\begin{algorithmic}[1] 
  \State Initialize $\{ \bm{\phi}^t, \bm{\psi}^t \}$ according to Eq.~\eqref{eqn:dual_decomposition_scores}
  \State Positive step-size $\alpha=\frac{1}{\max_{t \in \mathcal{T}} |V^t|}$
  \Repeat
    \For{$t \in \mathcal{T}$}
      \For{$l \in {1, \dots, L}$}
        \State Compute max-marginal via dyamic programming of state $l$ of vertex $i$ in tree $t$ as $\bm{\mu}_i^t(l)$    \label{alg:line:DP}
      \EndFor
      \State Compute optimizing state that minimizes $\bm{\mu}_i^t$ as $\bm{x}_i^{t*}$
    \EndFor
    \For{$t \in \mathcal{T}$}
      \For{$l \in {1, \dots, L}$}
        \State $\bm{\psi}_i^t(l) \leftarrow \bm{\psi}_i^t(l) - \alpha (\bm{\mu}_i^t(l) - \frac{1}{|\mathcal{T}(i)|} \sum_{\bar{t} \in \mathcal{T}(i)} \bm{\mu}_i^{\bar{t}}(l))$
      \EndFor
    \EndFor
    \Until $\bm{x}_i^{t*} = \bm{x}_i^{\bar{t}*}, \forall i \in V, \forall (t, \bar{t}) \in \mathcal{T}(i) \times \mathcal{T}(i)$
\end{algorithmic}
\label{alg:DD_fixed-point}
\end{algorithm}

There are several advantages of Alg~\ref{alg:DD_fixed-point}: 1) it is easily parallelizable with modern deep learning frameworks, as all sub-problems $t \in \mathcal{T}$ can be solved in parallel and all operations can be written in forms of \textit{sum} and \textit{max} of matrices. 2) Unlike subgradient-based methods~\cite{komodakis2007mrf,komodakis2011m3}, fixed-point update is fully sub-differentiable (it is not differentiable everywhere because of the usage of $max$ operators when computing the max-marginals. We will elaborate on this in section~\ref{sec:DDD_with_smoothed_max}), and thus the entire algorithm is
sub-differentiable except for the decoding of the final primal solution, which will be addressed in the next section.

\section{Decoding the Primal Results and Training}
\label{sec:decode_and_train}
Alg~\ref{alg:DD_fixed-point} does not always converge and it is thus desirable to be able to decode primal results given some intermediate max-marginals $\{\bm{\mu}^t\}$ with $\{ \bm{x}^{t*} \}$ that do not necessarily agree on each other.

A simple way to decode $\bm{x}_i^*$ given $\bm{\mu}_i^t(l), \forall l \in \{1, \dots L\}, \forall t \in \mathcal{T}(i)$ would be

\begin{align}
\label{eqn:primal_decode}
    \bm{x}_i^* = \argmax_{l \in \{1, \dots, L\}} \sum_{t \in \mathcal{T}(i)} \bm{\mu}_i^t(l)  
\end{align}

Of course, Eq.~\eqref{eqn:primal_decode} is neither differentiable nor sub-differentiable and is the only non-differentiable part of the whole dual-decomposition algorithm.  One simple solution would be using softmax on $\sum_{t \in \mathcal{T}(i)} \bm{\mu}_i^t$ (note $\bm{\mu}^t_i$ is a vector of length $L$) to output a probability distribution over $L$ labels:

\begin{align}
\label{eqn:softmax_decode}
    p(\bm{x}_i) = \frac{\exp(\sum_{t \in \mathcal{T}(i)} \bm{\mu}_i^t)}{\sum_{l' \in \{1, \dots, L\}} \exp(\sum_{t \in \mathcal{T}(i)} \bm{\mu}_i^t(l'))}
\end{align}

we can then use one-hot encoding of pixel-wise labels as ground-truth, and directly compute cross-entropy loss and gradients using softmax of the max-marginals, which is itself sub-differentiable and thus the entire system is end-to-end trainable with respect to CNN parameters $\bm{\theta}$.

\section{Differentiable Dual-Decomposition with Smoothed-Max Operator}
\label{sec:DDD_with_smoothed_max}
Note that line~\ref{alg:line:DP} of Alg.~\ref{alg:DD_fixed-point} is dynamic programming over trees. Formally, we can expand max-marginals of any $i \in V, t \in \mathcal{T}$ recursively as:

\begin{align}
\label{eqn:DP_max}
    \bm{\mu}_i^t(l) = \bm{\psi}_i^t(l) + \sum_{j \in \mathcal{C}(i)} \max_{l'} (\bm{\phi}_{ij}(l, l') + \bm{\mu}_j^t(l'))
\end{align}

where $\mathcal{C}(i)$ denotes the set of neighbors of node $i$.  Notice that the $max$ operator is only sub-differentiable; at first look this may not seem to be a problem as common activation functions in neural networks such as ReLU and leaky ReLU are also based on $max$ operation and are thus sub-differentiable.  However in practice the parameters for generating $\{\bm{\psi}, \bm{\phi}\}$ terms are often initialized from zero-mean uniform or gaussian distributions, which means that at the start of training $\{\bm{\psi}, \bm{\phi}\}$ are near-identical over classes and locations while $\max (\cdot)$ over such terms can be quite random.  In the backward pass, the gradient $\partial L / \partial \bm{\mu}^t_i (l)$ will only flow through the maximum of the $max$ operator, which, as we observed in our experiments, hinders the progress of learning drastically and often results in inferior training and testing performance.

To alleviate the aforementioned problem, we propose to employ the smoothed-max operator introduced in~\cite{DDP}.  Specifically, we implement the smoothed-max operator with negative-entropy regularization, whose forward-pass is the log-sum-exp operator while the backward-pass is softmax operator.  We denote $\bm{y} = \{y_1, \dots, y_L\}$ as input of $L$ real values, then we can define the forward pass and its gradient for the smoothed-max operator $\max_{\Omega}$ as:

\begin{align}
    \max\nolimits_{\Omega}(\bm{y}) &= \gamma \log \left( \sum_l \exp (y_l / \gamma) \right) \label{eqn:smoothed_max} \\
    \nabla \max\nolimits_{\Omega}(\bm{y}) &= \frac{\exp(\bm{y} / \gamma)}{\sum_l \exp (y_l / \gamma)} \label{eqn:smoothed_max_grad}
\end{align}
 
where $\gamma$ is a positive value that controls how strong the convex regularization is. Note that the gradient vector Eq.~\eqref{eqn:smoothed_max_grad} is just softmax over input values. This ensures that gradients from the loss layer can flow equally through input logits when inputs are near-identical, making the training process much easier at initialization.

An important result from~\cite{DDP} is that the negative-entropy regularized smoothed-max operator Eq.~\eqref{eqn:smoothed_max} satisfies associativity and distributivity, and for tree-structured graphs the smoothed-max-energy computed via recursive dynamic program is equivalent to the smoothed-max over the combinatorial space $\mathcal{X}^{\mathcal{G}}$.  When replacing standard max with smoothed-max Eq.~\eqref{eqn:smoothed_max}, Alg.~\ref{alg:DD_fixed-point} optimizes a convex upper-bound of Eq.~\eqref{eqn:MAP_dual_decomposition_reparam}, while both Lemma~\ref{lemm_1} and Theorem~\ref{theo_1} will still hold; we provide rigorous proof in Sec.~\ref{sec:app_monotone_updates_smoothed_max} as for why they hold.  We observe much faster convergence of Alg.~\ref{alg:DD_fixed-point} in practice when smoothed-max is used with a reasonable $\gamma$, \eg\ $\gamma=1.0$ or $2.0$.

%
%

\paragraph{Efficient Implementation of the Forward/Backward Pass} While the aforementioned dynamic programs can be implemented directly using basic operators of PyTorch~\cite{paszke2017automatic} along with its AutoGrad feature, we found it quite inefficient in practice, and opted to implement our own parallel dynamic program layer for computing horizontal/vertical marginals on pixel-grid (line~\ref{alg:line:DP} for Alg.~\ref{alg:DD_fixed-point}) which is up to 10x faster than a naive PyTorch implementation.  In Sec.~\ref{sec:app_parallel_dp} we will derive in detail how forward and backward passes for computing the max-marginals and its gradients on a $M \times M$ pixel grid can be efficiently implemented as parallel algorithm with time complexity of $\mathcal{O}(M |\mathcal{L}|)$, for both max-operator and smoothed-max operator. 

\paragraph{Relation to Marginal Inference}  We note that the objective of dual-decomposition with smoothed-max (Eq.~\eqref{eqn:MAP_dual_decomposition_smoothed_reparam} in Sec.~\ref{sec:app_monotone_updates_smoothed_max}) corresponds to tree-reweighted belief propagation (TRBP) objective that bounds the partition function (\ie\ , $Z$ in Eq.~\eqref{eqn:CRF_posterior}) with decomposition of graph into trees~\cite{Wainwright_2005_TRBP,Meltzer_2009_UAI,Jancsary_2011_AISTATS,Domke_2011_AAAI}.  For minimizing bounds on partition function, our approach optimizes a similar objective to~\cite{Domke_2011_AAAI} but we employ a monotone, highly parallel coordinate descent approach while~\cite{Domke_2011_AAAI} uses gradient-based approach.  Also, as observed in~\cite{Domke_2011_AAAI}, dual-decomposition represents no overhead when each edge is covered exactly once; but if more complicated tree bound needs to be used one could minimize over edge appearance probabilities via TRBP as described in~\cite{Wainwright_2005_TRBP}.

\section{Experiments}
\label{sec:experiments}
We evaluate our proposed approach on the PASCAL VOC 2012 semantic segmentation benchmark\cite{pascal-voc-2012}.  We use average pixel intersection-over-union (mIoU) of all foreground as the performance measure, as was used in most of the previous works.

\paragraph{PASCAL VOC} is the dataset widely used to benchmark for various computer vision tasks such as object detection and semantic image segmentation.  Its semantic image segmentation benchmark consists 21 classes including background.  The train, validation and test splits contain 1,464, 1,449 and 1,456 images respectively, with pixel-wise semantic labeling.  We also make use of additional annotations provided by SBD dataset~\cite{BharathICCV2011} to augment the original VOC 2012 \textit{train} split, resulting a final \textit{trainaug} set with 10,582 images.  We conduct our ablation study by training on \textit{trainaug} set and evaluate on \textit{val} set.

\paragraph{Implementation details} \quad We re-implement DeepLab V3~\cite{DeepLabV3} in PyTorch~\cite{paszke2017automatic} with block4-backbone (\ie\ no additional block after backbone except for one 1x1 convolution, two 3x3 convolutions and one final 1x1 convolution for classification) and Atrous Spatial Pyramid Pooling (ASPP) as baseline.  The output stride is set to 16 for all models.  We use ResNet-50~\cite{ResNet} and Xception-65~\cite{Xception} as backbones.
For training all models, we employ stochastic gradient descent (SGD) optimizer with momentum of 0.9 and ''poly" learning rate policy with intial learning rate 0.007 for $30k$ steps.
For training ResNet-based models, we use weight decay of 0.0001, while for training Xception-based models we use weight decay of 0.00004.  For data-augmentation we apply random color jittering, random scaling from 0.5 to 1.5, random horizontal flip, random rotation between -10\degree and 10\degree, and finally randomly crop $513 \times 513$ image patch of the transformed image.  Synchronized batch normalization with batch size of 16 is used for all models as standard for training semantic segmentation models.  For dual-decomposition augmented models we obtain the pairwise potentials by applying a fully-connected layer over concatenated features of pairs of locations on the feature map just before the final layer of baseline model, and run several steps of fixed-point iteration (FPI) of Alg.~\ref{alg:DD_fixed-point} using output logits from unary and pairwise heads for both training and inference, and keep all other configurations exactly the same as baselines.

\subsection{Results}
\label{sec:alblation_study}

\begin{table}
 \begin{center}
 \renewcommand{\tabcolsep}{3.0pt}
 \begin{tabular}{ l | c | c | c |}
 \hline
 Method & Backbone & \#+Params & mIoU \\ \hline
 baseline-block4~\cite{DeepLabV3} & ResNet-50 & 1.6M & 64.3 \\
 +15FPI & ResNet-50 & 1.8M & \textbf{68.6} \\
 +ASPP & ResNet-50 & 14.8M & 73.8 \\
 +ASPP+5FPI & ResNet-50 & 15.0M & \textbf{74.4} \\ \hline
 +ASPP & Xception-65 & 14.8M & 77.8 \\
 +ASPP+15FPI & Xception-65 & 15.0M & \textbf{78.0} \\ \hline
 \end{tabular}
 \end{center}
 \vspace{0.05in}
 \caption{Ablation study on PASCAL VOC \textit{val} set.  +ASPP indicates the backbone is augmented with atrous spatial pyramid pooling as described in~\cite{DeepLabV3}, +xFPI indicates the backbone is augmented with x iterations of fixed-point updates, and +ASPP+xFPI indicates both are applied.  \#+Params indicates number of new parameters other than those in the backbone.
 We achieve 50\% of the improvement of ASPP while introducing only 1/60 additional parameters compared to that of ASPP.  With ResNet-50 backbone, we also improve ASPP models when applying FPI on top of ASPP module.}
 \label{tab:ablation_result}
 \end{table}
 
\begin{figure*}
\begin{center}
\begin{tabular}{ c c c c }
Input Image & Ground Truth & DeepLabV3~\cite{DeepLabV3} & DeepLabV3+Ours \\ \hline
\raisebox{-.25\height}{\includegraphics[width=0.2\textwidth]{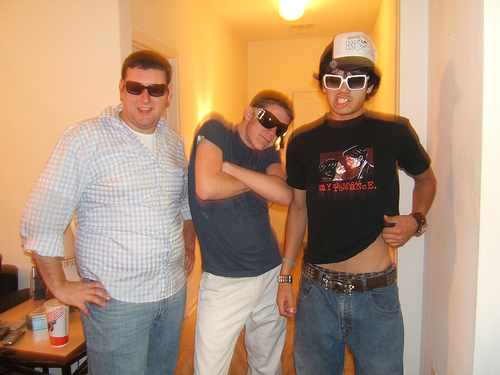}} &
\raisebox{-.25\height}{\includegraphics[width=0.2\textwidth]{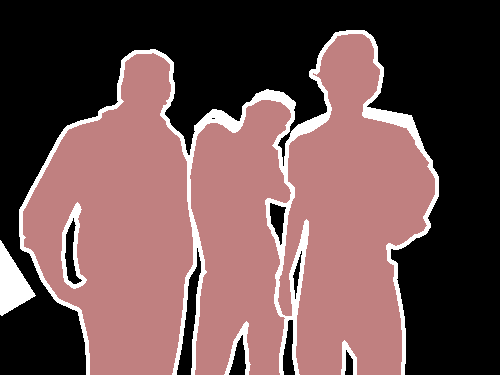}} &
\raisebox{-.25\height}{\includegraphics[width=0.2\textwidth]{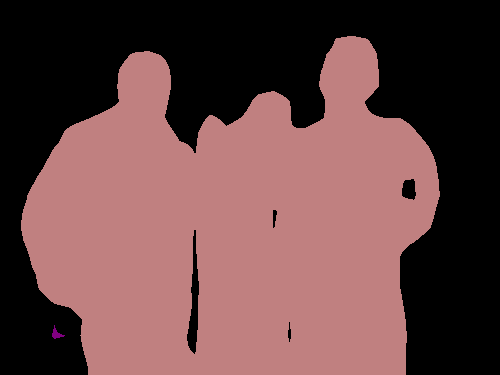}} &
\raisebox{-.25\height}{\includegraphics[width=0.2\textwidth]{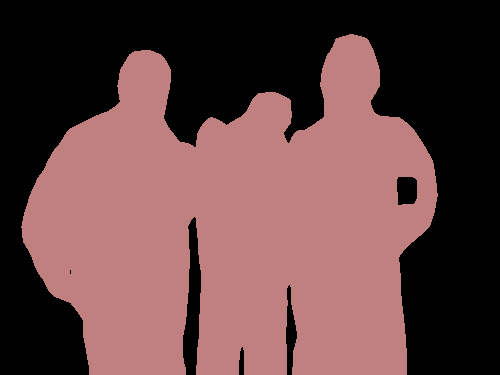}} \\
\\
\raisebox{-.25\height}{\includegraphics[width=0.2\textwidth]{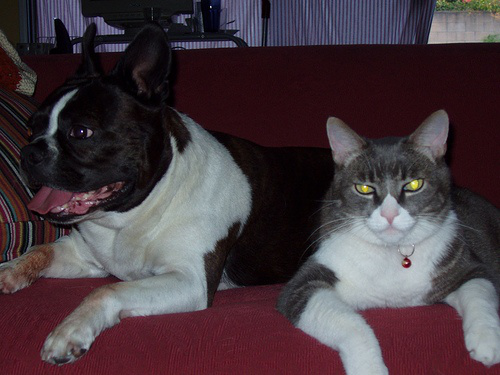}} &
\raisebox{-.25\height}{\includegraphics[width=0.2\textwidth]{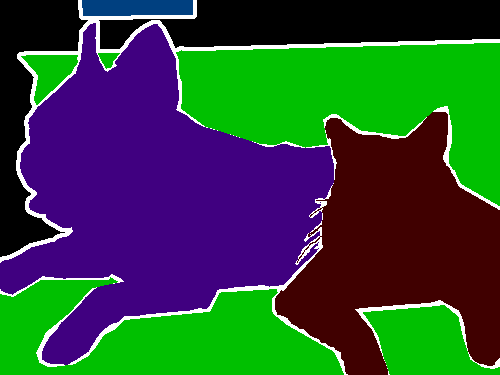}} &
\raisebox{-.25\height}{\includegraphics[width=0.2\textwidth]{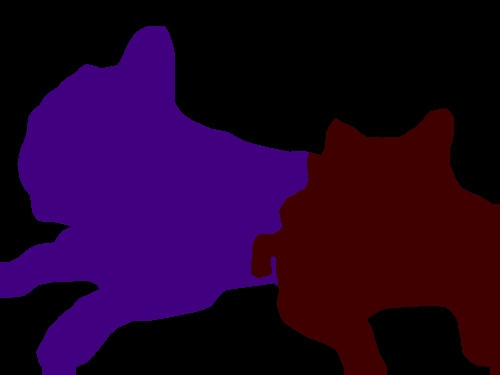}} &
\raisebox{-.25\height}{\includegraphics[width=0.2\textwidth]{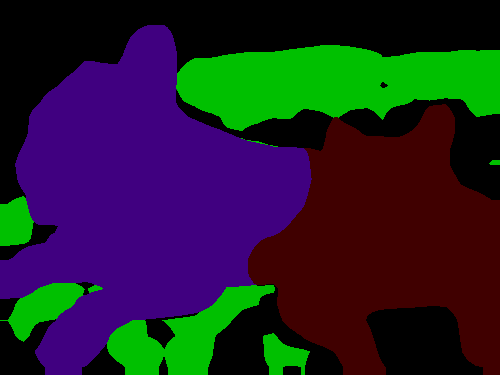}} \\
\\
\raisebox{-.25\height}{\includegraphics[width=0.2\textwidth]{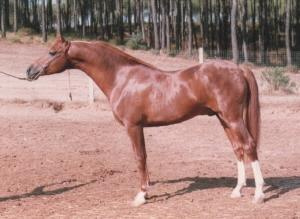}} &
\raisebox{-.25\height}{\includegraphics[width=0.2\textwidth]{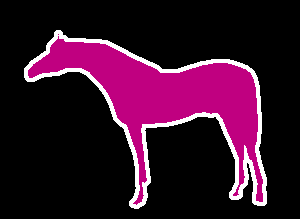}} &
\raisebox{-.25\height}{\includegraphics[width=0.2\textwidth]{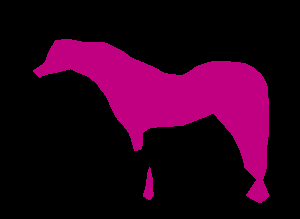}} &
\raisebox{-.25\height}{\includegraphics[width=0.2\textwidth]{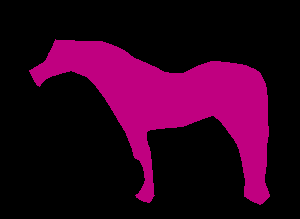}} \\
\end{tabular}
    \caption{Qualitative comparison of \cite{DeepLabV3} (\emph{third column}) and~\cite{DeepLabV3}+our approach (\emph{last column}) with Xception-65 backbone.  With our proposed CRF model and optimization, we obtain much more refined boundary especially in crowded scence (\emph{first row}), we are also able to encode inter-class (\emph{second row}) and intra-class (\emph{third row}) contextual information to boost the potentially weak unary classifier.}
\label{fig:qualitative_results}
\end{center}
\end{figure*}
 
We present experimental results on PASCAL VOC \textit{val} set in Table.~\ref{tab:ablation_result}.  It shows our proposed fixed-point update (FPI) consistently improves performance compared with pure CNN models.  FPI only adds about 0.2M additional parameters while providing 50\% of the improvement ASPP can provide.  It also provides additional 0.6\% and 0.2\% improvement in mIoU when applied after ASPP moduel for ResNet-50 and Xception-65, respecitvely.  More importantly, unlike the popular attention mechanisms~\cite{Wang2018non-local,Huang2019-CCAttn}, our FPI operates directly on the output potentials/logits of CNN instead of intermediate features of it, and works in a parameter-free manner, the additional parameters come from computing pairwise potential which are inputs to FPI.  Qualitative results are shown in Fig.~\ref{fig:qualitative_results}.

\section{Conclusion}
\label{sec:conclusion}
In this paper, we revisited dual-decomposition and derived novel monotone fixed-point updates that is amenable for parallel implementation, with detailed proofs on monotonicity of our update.  We also introduced the smoothed-max operator with negative-entropy regularization from~\cite{DDP} into our framework, resulting in a monotone and fully-differentiable inference method.
We also showed that how CNNs and CRFs can be jointly trained under our framework, improving semantic segmentation accuracy on the PASCAL VOC 2012 dataset over baseline models.  Future work includes extending our framework to more datasets and architectures, and tasks such as human-pose-estimation and stereo matching, all of which can be formulated as general CRF inference problems.  Another direction of research would be exploring unsupervised/semi-supervised learning by enforcing agreement of sub-problems.

{\small
\bibliographystyle{ieee_fullname}
\bibliography{DDD}

\begin{thebibliography}{10}\itemsep=-1pt

\bibitem{DDP}
Mathieu~Blondel Arthur~Mensch.
\newblock Differentiable dynamic programming for structured prediction and
  attention.
\newblock In {\em Proc. of ICML}, 2018.

\bibitem{belanger2017SPEN}
David Belanger, Bishan Yang, and Andrew McCallum.
\newblock End-to-end learning for structured prediction energy networks.
\newblock In {\em Proc. of ICML}, 2017.

\bibitem{DeepLabV3}
Liang{-}Chieh Chen, George Papandreou, Florian Schroff, and Hartwig Adam.
\newblock Rethinking atrous convolution for semantic image segmentation.
\newblock {\em CoRR}, abs/1706.05587, 2017.

\bibitem{Chen2015Structured}
Liang-Chieh Chen, Alexander~G. Schwing, Alan~L. Yuille, and Raquel Urtasun.
\newblock Learning deep structured models.
\newblock In {\em Proc. of ICML}, 2015.

\bibitem{Xception}
Franc¸ois Chollet.
\newblock Xception: Deep learning with depthwise separable convolutions.
\newblock In {\em Proc. of CVPR}, 2017.

\bibitem{Domke_2011_AAAI}
Justin Domke.
\newblock Dual decomposition for marginal inference.
\newblock In {\em Proc. of AAAI}, 2011.

\bibitem{pascal-voc-2012}
M. Everingham, L. Van~Gool, C.~K.~I. Williams, J. Winn, and A. Zisserman.
\newblock The {PASCAL} {V}isual {O}bject {C}lasses {C}hallenge 2012 {(VOC2012)}
  {R}esults.
\newblock
  http://www.pascal-network.org/challenges/VOC/voc2012/workshop/index.html.

\bibitem{Finley2008SSVM}
Thomas Finley and Thorsten Joachims.
\newblock Training structural svms when exact inference is intractable.
\newblock In {\em Proc. of ICML}, 2008.

\bibitem{globerson2008mplp}
Amir Globerson and Tommi~S. Jaakkola.
\newblock Fixing max-product: Convergent message passing algorithms for map
  lp-relaxations.
\newblock In {\em Proc. of NIPS}, 2008.

\bibitem{BharathICCV2011}
Bharath Hariharan, Pablo Arbelaez, Lubomir Bourdev, Subhransu Maji, and
  Jitendra Malik.
\newblock Semantic contours from inverse detectors.
\newblock In {\em Proc. of ICCV}, 2011.

\bibitem{ResNet}
Kaiming He, Xiangyu Zhang, Shaoqing Ren, and Jian Sun.
\newblock Deep residual learning for image recognition.
\newblock In {\em Proc. of CVPR}, 2016.

\bibitem{Huang2019-CCAttn}
Zilong Huang, Xinggang Wang, Lichao Huang, Chang Huang, Yunchao Wei, and Wenyu
  Liu.
\newblock Ccnet: Criss-cross attention for semantic segmentation.
\newblock In {\em Proc. of ICCV}, 2019.

\bibitem{Jancsary_2011_AISTATS}
Jeremy Jancsary and Gerald Matz.
\newblock Convergent decomposition solvers for trw free energies.
\newblock In {\em Proc. of AISTATS}, 2011.

\bibitem{Knobelreiter_2017_CVPR}
Patrick Knobelreiter, Christian Reinbacher, Alexander Shekhovtsov, and Thomas
  Pock.
\newblock End-to-end training of hybrid cnn-crf models for stereo.
\newblock In {\em Proc. of CVPR}, 2017.

\bibitem{TRW}
Vladimir Kolmogorov.
\newblock Convergent tree-reweighted message passing for energy minimization.
\newblock {\em IEEE transactions on pattern analysis and machine intelligence},
  28(10):1568--1583, 2006.

\bibitem{komodakis2011m3}
Nikos Komodakis.
\newblock Efficient training for pairwise or higher order crfs via dual
  decomposition.
\newblock In {\em Proc. of CVPR}, 2011.

\bibitem{komodakis2007mrf}
Nikos Komodakis, Nikos Paragios, and Georgios Tziritas.
\newblock Mrf optimization via dual decomposition: Message-passing revisited.
\newblock In {\em Proc. of ICCV}, 2007.

\bibitem{densecrf}
Philipp Kr{\"{a}}henb{\"{u}}hl and Vladlen Koltun.
\newblock Efficient inference in fully connected crfs with gaussian edge
  potentials.
\newblock In {\em Proc. of NIPS}, 2011.

\bibitem{piecewiseCRF}
Guosheng Lin, Chunhua Shen, Ian~D. Reid, and Anton van~den Hengel.
\newblock Efficient piecewise training of deep structured models for semantic
  segmentation.
\newblock In {\em Proc. of CVPR}, 2016.

\bibitem{DPN}
Ziwei Liu, Xiaoxiao Li, Ping Luo, Chen~Change Loy, and Xiaoou Tang.
\newblock Deep learning markov random field for semantic segmentation.
\newblock {\em IEEE transactions on pattern analysis and machine intelligence},
  40(8):1814--1828, 2018.

\bibitem{Wainwright_2005_TRBP}
Alan~Willsky Martin~Wainwright, Tommi S.~Jaakkola.
\newblock A new class of upper bounds on the log partition function.
\newblock {\em IEEE Transactions on Information Theory}, 51(7):2313--2335,
  2005.

\bibitem{paszke2017automatic}
Adam Paszke, Sam Gross, Soumith Chintala, Gregory Chanan, Edward Yang, Zachary
  DeVito, Zeming Lin, Alban Desmaison, Luca Antiga, and Adam Lerer.
\newblock Automatic differentiation in {PyTorch}.
\newblock In {\em NIPS Autodiff Workshop}, 2017.

\bibitem{MAP_NPhard}
Solomon~Eyal Shimony.
\newblock Finding maps for belief networks is np-hard.
\newblock {\em Artificial Intelligence}, 68(2):399--410, 1994.

\bibitem{Song_2019_ICCV}
Jie Song, Bjoern Andres, Michael~J. Black, Otmar Hilliges, and Siyu Tang.
\newblock End-to-end learning for graph decomposition.
\newblock In {\em Proc. of ICCV}, 2019.

\bibitem{sontag2011introduction}
David Sontag, Amir Globerson, and Tommi Jaakkola.
\newblock Introduction to dual decomposition for inference.
\newblock {\em Optimization for Machine Learning}, 1(219-254):1, 2011.

\bibitem{sontag2009tbcd}
David Sontag and Tommi Jaakkola.
\newblock Tree block coordinate descent for map in graphical models.
\newblock In {\em Proc. of AISTATS}, 2009.

\bibitem{Meltzer_2009_UAI}
Yair~Weiss Talya~Meltzer, Amir~Globerson.
\newblock Convergent message passing algorithms: a unifying view.
\newblock In {\em Proc. of UAI}, 2009.

\bibitem{Tsochantaridis2005SSVM}
Ioannis Tsochantaridis, Thorsten Joachims, Thomas Hofmann, and Yasemin Altun.
\newblock Large margin methods for structured and interdependent output
  variables.
\newblock {\em Journal of Machine Learning Research}, 6(2), 2005.

\bibitem{Wang2018non-local}
Xiaolong Wang, Ross Girshick, Abhinav Gupta, and Kaiming He.
\newblock Non-local neural networks.
\newblock In {\em Proc. of CVPR}, 2018.

\bibitem{yarkony2010covering}
Julian Yarkony, Charless Fowlkes, and Alexander Ihler.
\newblock Covering trees and lower-bounds on quadratic assignment.
\newblock In {\em Proc. of CVPR}, 2010.

\bibitem{CRFasRNN}
Shuai Zheng, Sadeep Jayasumana, Bernardino Romera{-}Paredes, Vibhav Vineet,
  Zhizhong Su, Dalong Du, Chang Huang, and Philip H.~S. Torr.
\newblock Conditional random fields as recurrent neural networks.
\newblock In {\em Proc. of ICCV}, 2015.

\end{thebibliography}
}

\appendix

\section{Monotone Fixed-Point Algorithm for Dual-Decomposition with Smoothed-Max}
\label{sec:app_monotone_updates_smoothed_max}
Following the notations in the main paper, we start by restating the typical dual-decomposition objective as:

\begin{align}
\label{eqn:MAP_dual_decomposition_reparam_dupe}
\min_{\{ \bm{\psi}^t \}} \max_{\{\bm{x}^t\}} \quad & \sum_{t \in \mathcal{T}} \left( \sum_{i \in V} \bm{x}_i^t \cdot \bm{\psi}_i^t + \sum_{ij \in E} \bm{x}_{ij}^t \cdot \bm{\phi}_{ij} \right) \\
\text{s.t.} \quad & \bm{x}^t \in \mathcal{X}^{\mathcal{G}}, \quad \forall t \in \mathcal{T} \nonumber \\
& \sum_{t \in \mathcal{T}(i)} \bm{\psi}^t_i = \bm{\psi}_i, \quad \forall i \in V \nonumber
\end{align}

Since $\bm{x}^t$'s are independent of each other $\forall t \in \mathcal{T}$, we can move the $max$ to the right of $\sum_{t \in \mathcal{T}}$:

\begin{align}
\label{eqn:MAP_dual_decomposition_reparam_independent_max}
\min_{\{ \bm{\psi}^t \}} \quad & \sum_{t \in \mathcal{T}} \max_{\bm{x}^t} \quad \left( \sum_{i \in V} \bm{x}_i^t \cdot \bm{\psi}_i^t + \sum_{ij \in E} \bm{x}_{ij}^t \cdot \bm{\phi}_{ij} \right) \\
\text{s.t.} \quad & \bm{x}^t \in \mathcal{X}^{\mathcal{G}}, \quad \forall t \in \mathcal{T} \nonumber \\
& \sum_{t \in \mathcal{T}(i)} \bm{\psi}^t_i = \bm{\psi}_i, \quad \forall i \in V \nonumber
\end{align}

Now we replace $max$ with smoothed-max defined as in Eq.~\eqref{eqn:smoothed_max}, we obtain the following new objective:

\begin{align}
\begin{split}
\label{eqn:MAP_dual_decomposition_smoothed_reparam}
\min_{\{ \bm{\psi}^t \}} \quad & \sum_{t \in \mathcal{T}} \gamma \log \\
&\sum_{ \bm{x}^t \in \mathcal{X}^{\mathcal{G}} } \exp \left( \frac{ \sum_{i \in V} \bm{x}_i^t \cdot \bm{\psi}_i^t + \sum_{ij \in E} \bm{x}_{ij}^t \cdot \bm{\phi}_{ij} }{\gamma} \right) \\
\text{s.t.} \quad & \sum_{t \in \mathcal{T}(i)} \bm{\psi}^t_i = \bm{\psi}_i, \quad \forall i \in V
\end{split}
\end{align}

Formally we define the smoothed-max marginal of state $l_i$ at location $i$ on sub-problem $t$ as ($\mathcal{C}(i)$ denotes the set of neighbors of node $i$):

\begin{align}
\label{eqn:smoothed_max_marginal}
\begin{split}
\bm{\nu}^t_i (l_i) =& \ \bm{\psi}^t_i (l_i) + \\
& \sum_{j \in \mathcal{C}(i)} \gamma \log \left( \sum_{l_j \in \mathcal{L}} \exp \left( \frac{\bm{\nu}^t_j (l_j) + \bm{\phi}^t_{ij} (l_i, l_j)}{\gamma} \right) \right)
\end{split}
\end{align}

And similar to max-marginal vector $\bm{\mu}^t_i$ for sub-problem $t$ location $i$ and max-energy $\bm{\mu}^t$ for sub-problem $t$, we define smoothed max-marginal vector $\bm{\nu}^t_i$ for sub-problem $t$ location $i$ and smoothed-max energy $\bm{\nu}^t$ for sub-problem $t$.  Specifically for smoothed-max energy we have:

\begin{align}
\label{eqn:smoothed_max_energy}
\bm{\nu}^t = \gamma \log \left( \sum_{l_i \in \mathcal{L}} \exp \left( \frac{\bm{\nu}^t_i(l_i)}{\gamma} \right) \right), \forall i \in V^t
\end{align}

Eq.~\eqref{eqn:smoothed_max_energy} is equivalent for any $i \in V^t$ because of an important conclusion from~\cite{DDP}, that the smoothed-max with negative-entropy regularization (\ie\ $logsumexp$) over the combinatorial space of a tree-structured problem equals to the smoothed-max-energy computed via dynamic programs with smoothed-max operator, this formally translates into:

\begin{align}
\label{eqn:smoothed_max_energy_equality}
\bm{\nu}^t = \gamma \log \sum_{ \bm{x}^t \in \mathcal{X}^{\mathcal{G}} } \exp \left( \frac{ \sum_{i \in V} \bm{x}_i^t \cdot \bm{\psi}_i^t + \sum_{ij \in E} \bm{x}_{ij}^t \cdot \bm{\phi}_{ij} }{\gamma} \right)
\end{align}

Now consider fixing the dual variables for all sub-problems at all locations except for those at one location $k$ and optimizing only with respect to the vector $\bm{\psi}^t_{k}, \forall t \in \mathcal{T}(k)$.  We now propose and prove a new monotone update rule for Eq.~\eqref{eqn:MAP_dual_decomposition_smoothed_reparam} under this circumstance.

\begin{lemma}
\label{lemm_1_smooth}
For a single location $k \in V$, the following coordinate update to $\bm{\psi}^t_{k}, \forall t \in \mathcal{T}(k)$ is optimal:
\begin{align}
\label{eqn:update_rule_single_smooth}
 \bm{\psi}_k^t(l_k) \leftarrow \bm{\psi}_k^t(l_k) - \left(\bm{\nu}_k^t(l_k) - \frac{1}{|\mathcal{T}(k)|} \sum_{\bar{t} \in \mathcal{T}(k)} \bm{\nu}_k^{\bar{t}}(l_k) \right), \nonumber \\
 \forall t \in \mathcal{T} (k), l_k \in \mathcal{L}
\end{align}
\end{lemma}

\begin{proof}
We want to optimize the following linear program with respect to $\bm{\psi}^t_k$ (note that $\bm{\nu}^t_k (l_k)$ is a function of $\bm{\psi}^t_k (l_k)$ according to Eq.~\eqref{eqn:smoothed_max_marginal}):

%

\begin{align}
\label{eqn:LP_dual_coordinate_descent_smooth_log}
\min_{\substack{\bm{\psi}^t_k \in \mathbb{R}^{|\mathcal{L}|}, \\ \forall t \in \mathcal{T}(k) } } \quad & \sum_{t \in \mathcal{T}(k)} \gamma \log \left( \sum_{l_k \in \mathcal{L}} \exp \left( \frac{\bm{\nu}^t_k(l_k)}{\gamma} \right) \right) \\
\text{s.t.} \quad & \sum_{t \in \mathcal{T}(k)} \bm{\psi}^t_k(l_k)= \bm{\psi}_k(l_k), \quad \forall l_k \in \mathcal{L} \nonumber
\end{align}

we can get rid of $\gamma \log (\cdot)$ since it's a monotonically increasing function (for positive $\gamma$), resulting in:

\begin{align}
\label{eqn:LP_dual_coordinate_descent_smooth}
\min_{\substack{\bm{\psi}^t_k \in \mathbb{R}^{|\mathcal{L}|}, \\ \forall t \in \mathcal{T}(k) } } \quad & \prod_{t \in \mathcal{T}(k)} \left( \sum_{l_k \in \mathcal{L}} \exp \left( \frac{\bm{\nu}^t_k(l_k)}{\gamma} \right) \right) \\
\text{s.t.} \quad & \sum_{t \in \mathcal{T}(k)} \bm{\psi}^t_k(l_k)= \bm{\psi}_k(l_k), \quad \forall l_k \in \mathcal{L} \nonumber
\end{align}

Let us define $\bm{\lambda}^t_k$ as the change of $\bm{\psi}^t_k$ after some update, and optimize over $\bm{\lambda}^t_k$ while fixing $\bm{\psi}^t_k$:

\begin{align}
\label{eqn:LP_dual_coordinate_descent_smooth_lambda}
\min_{\substack{\bm{\lambda}^t_k \in \mathbb{R}^{|\mathcal{L}|}, \\ \forall t \in \mathcal{T}(k)}} \quad & \prod_{t \in \mathcal{T}(k)} \left( \sum_{l_k \in \mathcal{L}} \exp \left( \frac{\bm{\nu}^t_k(l_k) + \bm{\lambda}^t_k(l_k)}{\gamma} \right) \right) \\
\text{s.t.} \quad & \sum_{t \in \mathcal{T}(k)} \bm{\lambda}^t_k(l_k)= 0, \quad \forall l_k \in \mathcal{L} \nonumber
\end{align}

In terms of $\bm{\lambda}^t_k (l_k)$, update rule Eq.~\eqref{eqn:update_rule_single_smooth} is equivalent to the solution:

\begin{align}
\label{eqn:solution_to_lambda}
\bm{\lambda}^t_k (l_k) = -\bm{\nu}^t_k (l_k) + \frac{1}{|\mathcal{T}(k)|} \sum_{\bar{t} \in \mathcal{T}(k)} \bm{\nu}^{\bar{t}}_k (l_k), \forall t \in \mathcal{T}(k)
\end{align}

This solution also makes $\bm{\nu}^t_k(l_k) + \bm{\lambda}^t_k(l_k) = \bm{\nu}^{\bar{t}}_k(l_k) + \bm{\lambda}^{\bar{t}}_k(l_k), \forall l_k \in \mathcal{L}, \forall (t, \bar{t}) \in \mathcal{T}(k) \times \mathcal{T}(k)$.  Thus if applying solution Eq.~\eqref{eqn:solution_to_lambda}, the equality condition of AM-GM inequality is met, resulting in:

\begin{align}
& \prod_{t \in \mathcal{T}(k)} \left( \sum_{l_k \in \mathcal{L}} \exp \left( \frac{\bm{\nu}^t_k(l_k) + \bm{\lambda}^t_k(l_k)}{\gamma} \right) \right) \nonumber \\
= & \left( \frac{1}{|\mathcal{T}(k)|} \sum_{t \in \mathcal{T}(k)} \sum_{l_k \in \mathcal{L}} \exp \left( \frac{\bm{\nu}^t_k(l_k) + \bm{\lambda}^t_k(l_k)}{\gamma} \right)  \right)^{|\mathcal{T}(k)|} \nonumber
\end{align}

When doing minimization, we can drop the denominator and exponent on the right-hand side, thus the objective becomes:

\begin{align}
\label{eqn:LP_dual_coordinate_descent_smooth_lambda_sum}
\min_{\substack{\bm{\lambda}^t_k \in \mathbb{R}^{|\mathcal{L}|}, \\ \forall t \in \mathcal{T}(k)}} \quad & \sum_{t \in \mathcal{T}(k)} \left( \sum_{l_k \in \mathcal{L}} \exp \left( \frac{\bm{\nu}^t_k(l_k) + \bm{\lambda}^t_k(l_k)}{\gamma} \right) \right) \\
\text{s.t.} \quad & \sum_{t \in \mathcal{T}(k)} \bm{\lambda}^t_k(l_k)= 0, \quad \forall l_k \in \mathcal{L} \nonumber
\end{align}

Since there is no constraint over different pairs in label space (\ie\ $\mathcal{L} \times \mathcal{L}$), it is equivalent to optimize for each unique $l_k \in \mathcal{L}$ independently, each optimization problem is:

\begin{align}
\label{eqn:LP_dual_coordinate_descent_smooth_lambda_sum_single}
\min_{\substack{\bm{\lambda}^t_k (l_k) \in \mathbb{R}, \\ \forall t \in \mathcal{T}(k)}} \quad & \sum_{t \in \mathcal{T}(k)} \exp \left( \frac{\bm{\nu}^t_k(l_k) + \bm{\lambda}^t_k(l_k)}{\gamma} \right) \\
\text{s.t.} \quad & \sum_{t \in \mathcal{T}(k)} \bm{\lambda}^t_k(l_k)= 0 \nonumber
\end{align}

Converting Eq.~\eqref{eqn:LP_dual_coordinate_descent_smooth_lambda_sum_single} to dual form:

\begin{align}
\begin{split}
\label{eqn:LP_dual_coordinate_descent_smooth_lambda_dual}
\min_{\substack{\bm{\lambda}^t_k (l_k) \in \mathbb{R}, \\ \forall t \in \mathcal{T}(k)}} \max_{\alpha \in \mathbb{R}} \quad & \sum_{t \in \mathcal{T}(k)} \exp \left( \frac{\bm{\nu}^t_k(l_k) + \bm{\lambda}^t_k(l_k)}{\gamma} \right) \\
& - \alpha \sum_{t \in \mathcal{T}(k)} \bm{\lambda}^t_k(l_k)
\end{split}
\end{align}

Note that we can put $min$ on the outside because an LP always satisfies KKT condition.  Setting derivative of Eq.~\eqref{eqn:LP_dual_coordinate_descent_smooth_lambda_dual} w.r.t. $\bm{\lambda}^t_k (l_k)$ to zero for any sub-problem $t$, we have:

\begin{align}
\alpha = \frac{1}{\gamma} \exp \left( \frac{\bm{\nu}^t_k(l_k) + \bm{\lambda}^t_k(l_k)}{\gamma} \right), \forall t \in \mathcal{T}(k)
\end{align}

Thus one must satisfy $\bm{\nu}^t_k(l_k) + \bm{\lambda}^t_k(l_k) = \bm{\nu}^{\bar{t}}_k(l_k) + \bm{\lambda}^{\bar{t}}_k(l_k), \forall (t, \bar{t}) \in \mathcal{T}(k) \times \mathcal{T}(k)$ which is exactly what Eq.~\eqref{eqn:solution_to_lambda} achieves.  Also remember that $\sum_{t \in \mathcal{T}(k)} \bm{\lambda}^t_k(l_k) = 0$ is satisfied when applying Eq.~\eqref{eqn:solution_to_lambda}, it follows that Eq.~\eqref{eqn:LP_dual_coordinate_descent_smooth_lambda_dual} = Eq.~\eqref{eqn:LP_dual_coordinate_descent_smooth_lambda_sum_single}, \ie\ duality gap of the LP is 0 and the update rule Eq.~\eqref{eqn:update_rule_single_smooth} is an optimal update for Eq.~\eqref{eqn:MAP_dual_decomposition_smoothed_reparam}


\end{proof}

We now prove the monotone update step for simultaneously updating all locations for dual-decomposition with smoothed-max:

\begin{theorem}
\label{theo_1_smooth}
The following update to $\bm{\psi}^t_i, \forall i \in V, t \in \mathcal{T}(i)$ will not increase the objective Eq.~\eqref{eqn:MAP_dual_decomposition_smoothed_reparam}:

\begin{align}
\label{eqn:update_rule_all_smooth}
 \bm{\psi}_i^t(l_i) \leftarrow \bm{\psi}_i^t(l_i) - \frac{1}{|V^0|} \left(\bm{\nu}_i^t(l_i) - \frac{1}{|\mathcal{T}(i)|} \sum_{\bar{t} \in \mathcal{T}(i)} \bm{\nu}_i^{\bar{t}}(l_i) \right), \nonumber \\
 \forall i \in V, t \in \mathcal{T} (i), l_i \in \mathcal{L}
\end{align}

where $|V^0| = \max_{t \in \mathcal{T}} |V^t|$ and $|V^t|$ denotes the number of vertices sub-problem $t$.
\end{theorem}

\begin{proof}
We denote $\hat{\bm{\nu}}^t$ as the smoothed-max-energy of sub-problem $t$ after we apply update Eq.~\eqref{eqn:update_rule_all_smooth} to $\{ \bm{\psi}^t \}$ and $\bar{\bm{\nu}}^t_i$ as smoothed-max-energy of sub-problem $t$ after we apply update Eq.~\eqref{eqn:update_rule_single_smooth} for location $i$.  Consider changes in objective from updating $\{ \bm{\psi}^t \}$ according to Eq.~\eqref{eqn:update_rule_all_smooth}:

\begin{align}
\label{eqn:change_in_smooth_dual_obj}
&- \sum_{t \in \mathcal{T}} \bm{\nu}^t + \sum_{t \in \mathcal{T}} \hat{\bm{\nu}}^t
\end{align}

We can deduct that $\bm{\nu}^t$ as defined by Eq.~\eqref{eqn:smoothed_max_energy_equality} (and consequently $\hat{\bm{\nu}}^t$) is a convex function of $\bm{\psi}^t$ by the following steps:
\begin{enumerate}
    \item $\sum_{i \in V} \bm{x}_i^t \cdot \bm{\psi}_i^t + \sum_{ij \in E} \bm{x}_{ij}^t \cdot \bm{\phi}_{ij}$ is a non-decreasing convex function of $\bm{\psi}^t$ for any $\bm{x}^t \in \mathcal{X}^{\mathcal{G}}$.
    \item $logsumexp$ function is non-decreasing convex function, and a composition of two non-decreasing convex functions is convex, thus $\bm{\nu}^t$ (and $\hat{\bm{\nu}}^t$) is a convex function of $\bm{\psi}^t$.
\end{enumerate}

We can then apply Jensen's Inequality to the second term of Eq.~\eqref{eqn:change_in_smooth_dual_obj}: 

\begin{align}
Eq.~\eqref{eqn:change_in_smooth_dual_obj} \leq &- \sum_{t \in \mathcal{T}} \bm{\nu}^t + \sum_{t \in \mathcal{T}} \left( \frac{|V^0| - |V^t|}{|V^0|} \bm{\nu}^t + \sum_{i \in V^t} \frac{1}{|V^0|} \bar{\bm{\nu}}^t_i \right) \nonumber \\
= & \frac{1}{|V^0|} \sum_{t \in \mathcal{T}} \sum_{i \in V^t} \left( \bar{\bm{\nu}}^t_i - \bm{\nu}^t \right) \nonumber \\
= & \frac{1}{|V^0|} \sum_{i \in V^t} \sum_{t \in \mathcal{T}(i)} \left( \bar{\bm{\nu}}^t_i - \bm{\nu}^t \right) \nonumber
\end{align}

Observe that $\sum_{t \in \mathcal{T}(i)} \bm{\nu}^t_i$ corresponds to objective Eq.~\eqref{eqn:LP_dual_coordinate_descent_smooth_log} before applying update Eq.~\eqref{eqn:update_rule_single_smooth} and $\sum_{t \in \mathcal{T}(i)} \bar{\bm{\nu}}^t$ corresponds the same objective after applying update Eq.~\eqref{eqn:update_rule_single_smooth}.  From Lemma~\ref{lemm_1_smooth} we know that the update Eq.~\eqref{eqn:update_rule_single_smooth} will not increase the objective Eq.~\eqref{eqn:LP_dual_coordinate_descent_smooth_log}, it follows that $\sum_{t \in \mathcal{T}(i)} \left( \bar{\bm{\nu}}^t_i - \bm{\nu}^t \right) \leq 0$, which results in Eq.~\eqref{eqn:change_in_smooth_dual_obj} $\leq$ 0 and thus the update Eq.~\eqref{eqn:update_rule_all_smooth} constitutes a non-increasing step to objective Eq.~\eqref{eqn:MAP_dual_decomposition_smoothed_reparam}.

\end{proof}

\section{Implementation Details for Parallel Dynamic Programming}
\label{sec:app_parallel_dp}

We define dynamic programming (DP) for computing max-marginal or smoothed-max-marginal term, $\bm{\mu}^t_i (l_i)$ or $\bm{\nu}^t_i (l_i)$, as follows:

\begin{align}
    \bm{\mu}_i^t(l_i) &= \bm{\psi}_i^t(l_i) + \sum_{j \in \mathcal{C}(i)} \max_{l_j \in \mathcal{L}} (\bm{\phi}_{ij}(l_i, l_j) + \bm{\mu}_j^t(l_j)) \label{eqn:DP_max_appendix}
\end{align}
\begin{align}
\begin{split}
    \bm{\nu}^t_i (l_i) =& \ \bm{\psi}^t_i (l_i) + \\
    &\sum_{j \in \mathcal{C}(i)} \gamma \log \left( \sum_{l_j \in \mathcal{L}} \exp \left( \frac{\bm{\nu}^t_j (l_j) + \bm{\phi}^t_{ij} (l_i, l_j)}{\gamma} \right) \right)
\end{split}
\label{eqn:DP_smoothed_max}
\end{align}

It is obvious that for any leaf nodes $k$ we have $\mathcal{C}(k) = \emptyset$.  In the following we derive algorithms for computing smoothed-max-marginals and their gradients, while max-marginals and their gradients can be derived in the same fashion.

%
%
%

Without loss of generality, we assume a $M \times M$ pixel-grid with stride 1 and stride 2 horizontal/vertical edges (as illustrated in Fig.~\ref{fig:graph_decomp_illust}) and want to compute smoothed-max-marginals for every label at every location.  We have $M$ horizontal chains and $M$ vertical chains of length $M$, respectively, and also:
\begin{enumerate}
    \item If $M$ is even: $2M$ horizontal chains and $2M$ vertical chains of length $\frac{M}{2}$.
    \item If $M$ is odd: $M$ horizontal chains and $M$ vertical chains of length $\left \lceil \frac{M}{2} \right \rceil$, $M$ horizontal chains and $M$ vertical chains of length $\left \lfloor \frac{M}{2} \right \rfloor$.
\end{enumerate}

It is straightforward that computing smoothed-max-marginals for each chain/sub-problem can be done in parallel.  A less straightforward point for parallelism is that for each sub-problem $t$ at a location $i$, we can compute $\bm{\nu}^t_i(l_i)$ for each $l_i \in \mathcal{L}$ in parallel, this requires the threads inside sub-problem $t$ to be synchronized at each location before proceeding to the next location.  To compute smoothed-max-marginals for all locations and all labels we need to run two passes of DP over each sub-problem.  The total time complexity is thus $\mathcal{O}(M |\mathcal{L}|)$, while the space complexity is $\mathcal{O}(M |\mathcal{L}|^2)$ ($|\mathcal{L}|^2$ for storing soft-max probabilities for later use in backward-pass, if no need to compute gradients then the space complexity becomes $\mathcal{O}(M |\mathcal{L}|)$).

For backward-pass, remember the loss function is computed with smoothed-max-marginals of \textit{all} locations and labels, thus we need to differentiate through \textit{all} locations and labels.  At first glance, this will result in an algorithm of $\mathcal{O}(M^2 |\mathcal{L}|)$ time complexity (with parallelization over locations and labels at each location) as gradient at one location is affected by gradient from all locations.  A better way of doing backward-pass would be starting from root/leaf location, passing gradients to the next location while also adding in the gradients of the next location from loss-layer, and recurse till the leaf/root.  We can do two passes of this DP to obtain the final gradients for every location, this results in a $\mathcal{O}(M |\mathcal{L}|)$ time and $\mathcal{O}(M |\mathcal{L}|^2)$ space algorithm.

Formally, we call DP from root to leaf as forward-DP and denote it with a subscript $f$, while DP from leaf to root is backward-DP and is denoted with subscript $b$.  In this context, $\mathcal{C}_f(i)$ indicates the set of previous node(s) of node $i$ in root-leaf direction, while $\mathcal{C}_b(i)$ indicates the set of previous node(s) of node $i$ in leaf-root direction.  The forward and backward passes for smoothed-max-marginals and max-marginals are described as Alg.~\ref{alg:DDP_forward_backward} and Alg.~\ref{alg:DP_forward_backward}, respectively.

\begin{algorithm*}
\caption{Forward and Backward passes for computing smoothed-max-marginals and their gradients on sub-problem $t$}
\begin{algorithmic}[1] 
  \Statex $\triangleright$ Forward-pass
  \Statex \textbf{Input}: potentials $\bm{\phi}^t, \bm{\psi}^t$ 
  \Statex \textbf{Output}: smoothed-max-marginals $\bm{\nu}^t_i (l_i), \forall i \in V^t, \forall l_i \in \mathcal{L}$ 
  \Statex \textbf{Initialize}: $\bm{\nu}^t_{f,i} (l_i) = \bm{\nu}^t_{b,i} (l_i) = 0, \forall i \in V^t, \forall l_i \in \mathcal{L}$ 
  \For{$i \in V^t$ from root to leaf} \Comment{\textbf{sequential}}
    \For{$l_i \in \mathcal{L}$} \Comment{\textbf{parallel}}
      \State $\bm{\nu}^t_{i,f} (l_i) = \bm{\psi}^t_i (l_i) + \sum_{j \in \mathcal{C}_f(i)} \gamma \log \left( \sum_{l_j \in \mathcal{L}} \exp \left( \frac{\bm{\nu}^t_{j,f} (l_j) + \bm{\phi}^t_{ij} (l_i, l_j)}{\gamma} \right) \right)$
      \For{$j \in \mathcal{C}_f(i)$} \Comment{\textbf{parallel}}
        \For{$l_j \in \mathcal{L}$} \Comment{\textbf{parallel}}
          \State $\bm{w}_{ji, f} (l_j, l_i) = \frac{\exp ((\bm{\nu}^t_{j,f} (l_j) + \bm{\phi}^t_{ij} (l_i, l_j)) / \gamma )}{\sum_{l \in \mathcal{L}} \exp ((\bm{\nu}^t_{j,f} (l) + \bm{\phi}^t_{ij} (l_i, l)) / \gamma )}$ \Comment{\textbf{store weight}}
        \EndFor
      \EndFor
    \EndFor
  \EndFor
  \For{$i \in V^t$ from leaf to root} \Comment{\textbf{sequential}}
    \For{$l_i \in \mathcal{L}$} \Comment{\textbf{parallel}}
      \State $\bm{\nu}^t_{i,b} (l_i) = \bm{\psi}^t_i (l_i) + \sum_{j \in \mathcal{C}_b(i)} \gamma \log \left( \sum_{l_j \in \mathcal{L}} \exp \left( \frac{\bm{\nu}^t_{j,b} (l_j) + \bm{\phi}^t_{ij} (l_i, l_j)}{\gamma} \right) \right)$
      \For{$j \in \mathcal{C}_b(i)$} \Comment{\textbf{parallel}}
        \For{$l_j \in \mathcal{L}$} \Comment{\textbf{parallel}}
          \State $\bm{w}_{ij, b} (l_i, l_j) = \frac{\exp ((\bm{\nu}^t_{j,b} (l_j) + \bm{\phi}^t_{ij} (l_i, l_j)) / \gamma )}{\sum_{l \in \mathcal{L}} \exp ((\bm{\nu}^t_{j,b} (l) + \bm{\phi}^t_{ij} (l_i, l)) / \gamma )}$ \Comment{\textbf{store weight}}
        \EndFor
      \EndFor
      \State $\bm{\nu}_i^t(l_i) = \bm{\nu}_{i,f}^t(l_i) + \bm{\nu}_{i,b}^t(l_i) - \bm{\psi}^t_i (l_i)$
    \EndFor
  \EndFor
  \Statex $\triangleright$ Backward-pass
  \Statex \textbf{Input}: gradients w.r.t. max-marginals, $\nabla \bm{\nu}^t_i (l_i), \forall i \in V^t, \forall l_i \in \mathcal{L}$, and the $\bm{w}_f, \bm{w}_b$ terms from forward-pass 
  \Statex \textbf{Output}: gradients w.r.t. potentials $\nabla \bm{\phi}^t, \nabla \bm{\psi}^t$ 
  \Statex \textbf{Initialize}: $\nabla \bm{\phi}^t = \bm{0}, \nabla \bm{\psi}^t_{f,i} (l_i) = \nabla \bm{\psi}^t_{b,i} (l_i) = \nabla \bm{\nu}^t_i (l_i), \forall i \in V^t, \forall l_i \in \mathcal{L}$ 
  \For{$i \in V^t$ from leaf to root} \Comment{\textbf{sequential}}
    \For{$l_i \in \mathcal{L}$} \Comment{\textbf{parallel}}
      \For{$j \in \mathcal{C}_b(i)$} \Comment{\textbf{sequential}}
        \For{$l_j \in \mathcal{L}$} \Comment{\textbf{sequential}}
          \State $\nabla \bm{\psi}_{i,b}^t(l_i) = \nabla \bm{\psi}_{i,b}^t(l_i) + \bm{w}_{ij, f} (l_i, l_j) \nabla \bm{\psi}_{j,b}^t(l_j)$ \Comment{\textbf{back-track}}
          \State $\nabla \bm{\phi}_{ij}^t (l_i, l_j) = \nabla \bm{\phi}_{ij}^t (l_i, l_j) + \bm{w}_{ij, f} (l_i, l_j) \nabla \bm{\psi}_{j,b}^t(l_j)$
        \EndFor
      \EndFor
    \EndFor
  \EndFor
  \For{$i \in V^t$ from root to leaf} \Comment{\textbf{sequential}}
    \For{$l_i \in \mathcal{L}$} \Comment{\textbf{parallel}}
      \For{$j \in \mathcal{C}_f(i)$} \Comment{\textbf{sequential}}
        \For{$l_j \in \mathcal{L}$} \Comment{\textbf{sequential}}
          \State $\nabla \bm{\psi}_{i,f}^t(l_i) = \nabla \bm{\psi}_{i,f}^t(l_i) + \bm{w}_{ji, b} (l_j, l_i) \nabla \bm{\psi}_{j,f}^t(l_j)$ \Comment{\textbf{back-track}}
          \State $\nabla \bm{\phi}_{ij}^t (l_i, l_j) = \nabla \bm{\phi}_{ij}^t (l_i, l_j) + \bm{w}_{ji, b} (l_j, l_i) \nabla \bm{\psi}_{j,f}^t(l_j)$
        \EndFor
      \EndFor
      \State $\nabla \bm{\psi}_i^t(l_i) = \nabla \bm{\psi}_{i,f}^t(l_i) + \nabla \bm{\psi}_{i,b}^t(l_i) - \nabla \bm{\nu}_i^t(l_i)$
    \EndFor
  \EndFor
\end{algorithmic}
\label{alg:DDP_forward_backward}
\end{algorithm*}

\begin{algorithm*}
\caption{Forward and Backward passes for computing max-marginals and their gradients on sub-problem $t$}
\begin{algorithmic}[1] 
  \Statex $\triangleright$ Forward-pass
  \Statex \textbf{Input}: potentials $\bm{\phi}^t, \bm{\psi}^t$ 
  \Statex \textbf{Output}: max-marginals $\bm{\mu}^t_i (l_i), \forall i \in V^t, \forall l_i \in \mathcal{L}$ 
  \Statex \textbf{Initialize}: $\bm{\mu}^t_{f,i} (l_i) = \bm{\mu}^t_{b,i} (l_i) = 0, \forall i \in V^t, \forall l_i \in \mathcal{L}$ 
  \For{$i \in V^t$ from root to leaf} \Comment{\textbf{sequential}}
    \For{$l_i \in \mathcal{L}$} \Comment{\textbf{parallel}}
      \State $\bm{\mu}_{i,f}^t(l_i) = \bm{\psi}_i^t(l_i) + \sum_{j \in \mathcal{C}_f(i)} \max_{l_j \in \mathcal{L}} (\bm{\phi}_{ij}(l_i, l_j) + \bm{\mu}_{j,f}^t(l_j))$
      \For{$j \in \mathcal{C}_f(i)$} \Comment{\textbf{parallel}}
        \State $\bm{r}_{ji,f} (l_i) = \argmax_{l_j \in \mathcal{L}} \bm{\phi}_{ij}(l_i, l_j) + \bm{\mu}_{j,f}^t(l_j) $ \Comment{\textbf{store index}}
      \EndFor
    \EndFor
  \EndFor
  \For{$i \in V^t$ from leaf to root} \Comment{\textbf{sequential}}
    \For{$l_i \in \mathcal{L}$} \Comment{\textbf{parallel}}
      \State $\bm{\mu}_{i,b}^t(l_i) = \bm{\psi}_i^t(l_i) + \sum_{j \in \mathcal{C}_b(i)} \max_{l_j \in \mathcal{L}} (\bm{\phi}_{ij}(l_i, l_j) + \bm{\mu}_{j,b}^t(l_j))$
      \For{$j \in \mathcal{C}_b(i)$} \Comment{\textbf{parallel}}
        \State $\bm{r}_{ij,b} (l_i) = \argmax_{l_j \in \mathcal{L}} \bm{\phi}_{ij}(l_i, l_j) + \bm{\mu}_{j,b}^t(l_j) $ \Comment{\textbf{store index}}
      \EndFor
    \EndFor
  \EndFor
  \Statex $\triangleright$ Backward-pass
  \Statex \textbf{Input}: gradients w.r.t. max-marginals, $\nabla \bm{\mu}^t_i (l_i), \forall i \in V^t, \forall l_i \in \mathcal{L}$ 
  \Statex \textbf{Output}: gradients w.r.t. potentials $\nabla \bm{\phi}^t, \nabla \bm{\psi}^t$ 
  \Statex \textbf{Initialize}: $\nabla \bm{\phi}^t = \bm{0}, \nabla \bm{\psi}^t_{f,i} (l_i) = \nabla \bm{\psi}^t_{b,i} (l_i) = \nabla \bm{\mu}^t_i (l_i), \forall i \in V^t, \forall l_i \in \mathcal{L}$ 
  \For{$i \in V^t$ from leaf to root} \Comment{\textbf{sequential}}
    \For{$l_i \in \mathcal{L}$} \Comment{\textbf{parallel}}
      \For{$j \in \mathcal{C}_b(i)$} \Comment{\textbf{sequential}}
        \For{$l_j \in \mathcal{L}$} \Comment{\textbf{sequential}}
            \If{$l_i = \bm{r}_{ij,f} (l_j)$}
            \State $\nabla \bm{\psi}_{i,b}^t(l_i) = \nabla \bm{\psi}_{i,b}^t(l_i) + \nabla \bm{\psi}_{j,b}^t(l_j)$ \Comment{\textbf{back-track}}
            \State $\nabla \bm{\phi}_{ij}^t (l_i, l_j) = \nabla \bm{\phi}_{ij}^t (l_i, l_j) + \nabla \bm{\psi}_{j,b}^t(l_j)$
          \Else
            \State do nothing
          \EndIf
        \EndFor
      \EndFor
    \EndFor
  \EndFor
  \For{$i \in V^t$ from root to leaf} \Comment{\textbf{sequential}}
    \For{$l_i \in \mathcal{L}$} \Comment{\textbf{parallel}}
      \For{$j \in \mathcal{C}_f(i)$} \Comment{\textbf{sequential}}
        \For{$l_j \in \mathcal{L}$} \Comment{\textbf{sequential}}
          \If{$l_i = \bm{r}_{ji,b} (l_j)$}
            \State $\nabla \bm{\psi}_{i,f}^t(l_i) = \nabla \bm{\psi}_{i,f}^t(l_i) + \nabla \bm{\psi}_{j,f}^t(l_j)$ \Comment{\textbf{back-track}}
            \State $\nabla \bm{\phi}_{ij}^t (l_i, l_j) = \nabla \bm{\phi}_{ij}^t (l_i, l_j) + \nabla \bm{\psi}_{j,f}^t(l_j)$
          \Else
            \State do nothing
          \EndIf
        \EndFor
      \EndFor
      \State $\nabla \bm{\psi}_i^t(l_i) = \nabla \bm{\psi}_{i,f}^t(l_i) + \nabla \bm{\psi}_{i,b}^t(l_i) - \nabla \bm{\mu}_i^t(l_i)$
    \EndFor
  \EndFor
\end{algorithmic}
\label{alg:DP_forward_backward}
\end{algorithm*}


\end{document}